 \documentclass[a4paper]{article}

\usepackage{makeidx}  
\usepackage{amssymb}
\usepackage{latexsym}
\usepackage{amsbsy}
\usepackage{amsmath}

\usepackage{enumerate}

\usepackage{natbib}
\setcitestyle{authoryear,round,citesep={;},aysep={,},yysep={;}}
\renewcommand{\cite}[1]{\citep{#1}}

\usepackage{graphicx} 
\usepackage{subfigure}

\usepackage{color} 
\definecolor{mydarkblue}{rgb}{0,0.08,0.45}
\usepackage{hyperref}
\hypersetup{ %
  colorlinks=true,
  linkcolor=mydarkblue,
  citecolor=mydarkblue,
  filecolor=mydarkblue,
  urlcolor=mydarkblue}
\usepackage[normalem]{ulem}

\usepackage{amssymb}
\usepackage{amsmath}
\usepackage{amsthm}
\usepackage{graphicx}
\usepackage{subfigure}
\usepackage{xspace}
\usepackage[table]{xcolor}
\usepackage{booktabs}
\usepackage{pgf}
\usepackage{url}

\renewcommand{\mapsto}{\to}

\newcommand{\Hcal}{\mathcal{H}}

\newcommand{\Hcalbf}{{\pmb\Hcal}}

\newcommand{\hbf}{{\bf h}}

\newcommand{\imY}{\ima\Ybf}
\newcommand{\hullimY}{\conv{\imY}}

\newcommand{\posterior}{\rho}
\newcommand{\prior}{\pi}

\newcommand{\xbf}{{\bf x}}
\newcommand{\xb}{\xbf}
\newcommand{\ex}{(\xbf,y)}
\newcommand{\exbf}{(\xbf,\ybf)}

\newcommand{\Y}{\Ybf}
\newcommand{\X}{{\bf X}}
\newcommand{\Xbf}{\X}

\newcommand{\ybf}{\mathbf{y}}
\newcommand{\cbf}{\mathbf{c}}
\newcommand{\Ybf}{{\bf Y}}

\newcommand{\I}{\mathrm{\bf I}}

\newcommand{\sign}{\operatorname{sign}}

\newcommand{\R}{\mathbb{R}}
\newcommand{\Risk}{\mathbf{R}}

\newcommand{\RP}{\Risk_{D}}

\newcommand{\BQ}{\BQs} 
\newcommand{\GQ}{{\bf G}_{\posterior}}
\newcommand{\BQs}{{\bf B}_{\posterior}}

\DeclareMathOperator*{\Prob}{\mathrm{\mathbf{Pr}}}
\DeclareMathOperator*{\Esp}{\mathrm{\mathbf{E}}}
\DeclareMathOperator*{\Var}{\mathrm{\mathbf{Var}}}
\DeclareMathOperator*{\argmin}{\mathrm{argmin}}
\DeclareMathOperator*{\argmax}{\mathrm{argmax}}
\DeclareMathOperator*{\ima}{\mathrm{Im}}
\DeclareMathOperator*{\cv}{\mathrm{conv}}
\newcommand{\expp}{\mathrm{e}}

\newcommand{\e}[1]{\expp^{#1}}
\newcommand{\eqdef}{=}
\newcommand{\esp}[1]{\Esp_{#1}}
\newcommand{\conv}[1]{\cv(#1)}
\newcommand{\espdevant}[1]{{\textstyle\Esp_{#1}\,}}

\newcommand{\LB}{\left[}
\newcommand{\RB}{\right]}

\newcommand{\gc}[2]{\Esp_{\hbf\sim \posterior}\hbf_{#2}(#1)}

\newcommand{\MQ}{\MQs}
\newcommand{\MQs}{{\bf M}_\posterior}
\newcommand{\MQprimes}{{\bf M}_{\posterior'}}

\newcommand{\Mbf}{{\bf M}}

\newcommand{\Q}{\mbox{\tiny$Q$}}
\newcommand{\MQP}{M_\posterior^D}

\newcommand{\OQ}{\OQs}
\newcommand{\OQtwo}{{\bf M}_{\posterior,\!\frac12}}
\newcommand{\OQbof}{{\bf M}_{\posterior,\!\mbox{\tiny$\frac{Q\!-\!1}{Q}$}}} 

\newcommand{\OQtwos}{{\bf M}_{\posterior,\!\frac12}}
\newcommand{\OQbofs}{{\bf M}_{\posterior,\!\mbox{\tiny$\frac{Q\!-\!1}{Q}$}}}
\newcommand{\OQs}{{\bf M}_{\posterior,\omega}}

\newcommand{\OQQ}{{\bf M}_{\posterior,\!\frac1Q}}

\newcommand{\mfracq}{\frac{-1}{\sqrt{2}}}

\newcommand{\fracq}{\frac{1}{\sqrt{2}}}
\newcommand{\Cbound}{\mbox{$\mathcal{C}$-bound}\xspace}
\newcommand{\Cbounds}{\mbox{$\mathcal{C}$-bounds}\xspace}

\newcommand{\KL}{{\rm KL}}
\newcommand{\kl}{{\rm kl}}

\usepackage{amsthm}
\newtheorem{cor}{Corollary}

\newtheorem{definition}{Definition}
\newtheorem{lemma}{Lemma}
\newtheorem{theorem}{Theorem}

\newtheorem{remark}{Remark}

\usepackage{fancyhdr}
\usepackage{fancybox}
\newcommand{\ver}{2}
\begin{document}
 
\title{On the Generalization of the C-Bound\\to Structured Output Ensemble Methods}
\author{Fran\c cois Laviolette$^1$ \and  Emilie Morvant$^2$ \and  Liva Ralaivola$^3$\and Jean-Francis Roy$^1$ \and 
$\mbox{}^1$ D\'epartement d'informatique et de g\'enie logiciel, Universit\'e Laval, Qu\'{e}bec, Canada
\and
$\mbox{}^2$ Laboratoire Hubert Curien, Universit\'e Jean Monnet, UMR CNRS 5516, Saint-Etienne, France
\and
$\mbox{}^3$ Aix-Marseille Univ., LIF-QARMA, CNRS, UMR 7279, F-13013, Marseille, France}
\maketitle

\begin{abstract}
	This paper generalizes an important result from the PAC-Bayesian literature for binary classification to the 
	case of ensemble methods for structured outputs. We prove a generic version of the \Cbound, an upper bound over the risk of models expressed as a weighted majority vote that is based on the first and second statistical moments of the vote's margin. This bound may advantageously $(i)$ be applied on more complex outputs such as multiclass labels and multilabel,
 and $(ii)$ allow to consider margin relaxations. These results open the way to develop new ensemble methods for structured output prediction with PAC-Bayesian guarantees.
\end{abstract}

\section{Introduction}
\label{sec:introduction}

It is well-known that learning predictive models capable of dealing with outputs that are richer than binary outputs 
({\it e.g.,} multiclass or multilabel)
and for which theoretical guarantees exist is still
a realm of intensive investigations.
From a practical standpoint, a lot of relaxations for learning with complex outputs have been devised. 
A common approach consists in decomposing the output space into ``simpler'' spaces so that the
learning problem at hand can be reduced to a few easier ({\it i.e.,} binary) learning tasks.
For instance, this is the idea spurred by the Error-Correcting Output Codes~\cite{dietterich1995solving} that makes possible to reduce multiclass or multilabel problems into binary classification tasks,{\it e.g.}, \cite{allwein2001reducing,mroueh2012multiclass,read2011classifier,TsoumakasV07,ICML2012Zhang_778}. In our work, we study the problem of complex output prediction by focusing on prediction functions that take the form of a weighted majority vote over a set of complex output classifiers (or voters).  Recall that \emph{ensemble methods} can all be seen as majority vote learning procedures \citep{Dietterich00,re2012ensemble}. 
Methods such as Bagging \citep{b-96}, Boosting \citep{schapire99} and Random Forests \citep{randomforests} are representative voting methods. \citet{ckm-2014} have proposed various ensemble methods for the structured output prediction framework.

Note also that majority votes are also central to the Bayesian approach \citep{gelman2004bayesian} with the notion of Bayesian model averaging \citep{Domingos00,HausslerKS94} and most of kernel-based predictors, such as the Support Vector Machines \cite{boser1992training,CortesV95} may be viewed as weighted majority votes as well: for binary classification, where the predicted class for some input $\xb$ is computed as the sign of $\sum_i \alpha_i\, \ybf_i\, k(\xb_i,\xb)$, each voter is simply given by $\ybf_ik(\xb_i,\cdot)$. 

From a theoretical standpoint, as far as binary classification is concerned, 
the notion of {\em margin} is often the crux to establish the generalization
ability of a majority vote predictor. For instance, still considering the binary case,
the margin of a majority vote realized on an example  is defined as
the difference between the total weight of the voters that predicted the correct class minus the total weight given to the incorrect one. In the PAC-Bayesian analysis, which is our working setup, one way to provide generalization bounds for a majority vote classifier is to relate it to a stochastic classifier, the \emph{Gibbs} classifier, whose risk is the
weighted risk of the individual voters involved in the majority vote.
Up to a linear transformation, the Gibbs risk is equivalent to the first statistical moment of the margin \cite{MinCq}.
This PAC-Bayesian analysis can be very accurate when the Gibbs risk is low, as in the situation where the voters having large weights are performing well as done by \citet{gllm-09}, \citet{ls-03}, together with \citet{mcallester2009generalization} for the structured output framework.
However, for ensemble methods, it is not unusual to be in the situation where, on the one hand, the voters achieve performances only slightly
above the chance level---which makes it impossible to find weights that induce a small Gibbs risk---and, on the other hand, the risk of the majority vote itself is very low. Hence, to better capture the accuracy of a majority vote in a PAC-Bayesian fashion, it is required to consider more than the Gibbs risk, {\it i.e.}, more than only the first statistical moment of the margin. This idea, which has been studied in the context of ensemble methods by~\cite{Blanchard04,randomforests}, has been revisited as the \Cbound by~\citet{Lacasse07} in the PAC-Bayesian setting.
This bound sheds light on an essential feature of weighted majority votes: how good the voters individually are is just as important as how correlated their predictions are. This has inspired a new ensemble method for binary classification with PAC-Bayesian generalization guarantees named MinCq \citep{MinCq}, whose performances are state-of-the-art. 
In the multiclass setting, there exists one PAC-Bayesian bound, 
but it is based on the confusion matrix of the Gibbs classifier~\citep{MorvantKR12}. 
\citet{kuznetsov2014multi} have recently proposed an improved Rademacher bound for multiclass prediction that is based on the notion of the multiclass margin of~\citet{randomforests} (Definition~\ref{def:notre_marge} in the present paper). 
However, as for the binary case, these bounds suffer from the same lack of tightness when the voters of the majority vote perform poorly. 

Here, we intend to generalize the \Cbound to more complex situations.
We first propose a formulation of the \Cbound for ensemble methods for complex output settings, that makes it possible for all the binary classification-based 
results of \citet{Lacasse07} and \citet{MinCq} 
to be generalized. 
Since for complex output prediction
the usual margin relies on the maximal deviation between the total weight given to the true class minus the  maximal total weight given to the ``runner-up'' incorrect one,
we base our theory 
on a general notion of margin allowing us to consider extensions of the usual margin.
Moreover, similarly as for binary classification \citep{Lacasse07,MinCq},  we derive a PAC-Bayesian generalization bound and show how we can estimate such \Cbounds from a sample.
In light of this general theoretical result, we propose two specializations 
suitable for multiclass classification with ensemble methods based on the true margin and on a relaxation that we call $\omega$-margin. We highlight the behavior of these \Cbounds through an empirical study. Finally, we instantiate it to multilabel prediction problems. 
We see these theoretical results as a first step towards the design of new and well-founded learning algorithms for  complex outputs.

This paper is organized as follows.
Section~\ref{sec:cborne} recalls the binary \Cbound, that is  generalized to a more general setting in Section~\ref{sec:cborne_generale}. 
We then specialize this bound to multiclass prediction in Section~\ref{sec:cborne_multi}, and to multilable prediction in Section~\ref{sec:cborne_multilabel}. 
We conclude in Section~\ref{sec:conclusion}.

\section{Ensemble Methods in Binary Classification}
\label{sec:cborne}
\label{sec:cborne_binaire}

For binary classification with majority vote-based ensemble methods, we often consider an
arbitrary input space $\Xbf$, an output space of two classes \mbox{$\Ybf  =  \{-1,+1\}$}, and a set $\Hcalbf\subseteq\{\hbf:\Xbf \to [-1,+1]\}$ of \emph{voters}. 
A voter can return
any  value in $[-1,+1]$, interpretable
as a level of confidence of the voter into its option which is $+1$ if the output is positive and $-1$ otherwise. A voter that always outputs values in $\{-1,+1\}$ is a \emph{(binary) classifier}.
The \emph{binary \mbox{$\posterior$-weighted} majority vote} $\BQs(\cdot)$ is the classifier returning
either of the two options that  has obtained
the larger weight in the vote, {\it i.e.}:
$$
\forall \xbf  \in   \X,\quad \BQs(\xbf)  \  \eqdef\   \argmax_{\ybf\in{\{-1,+1\}}}\ \esp{\hbf\sim \posterior} \left(\, \Big|\hbf(\xbf) \Big| \ \I\Big[\sign(\hbf(\xbf))=\ybf\Big]\, \right) 
\  = \    \sign\left[\esp{\hbf\sim \posterior} \hbf(\xbf)\right]\,,
$$
where $\I[a] = 1$ if predicate $a$ is true and $0$ otherwise.\\
Given a training set $S$ of observed data in which each example $(\xbf, \ybf) \in  S$ is drawn {\em i.i.d.} from a fixed yet unknown probability distribution $D$ on \mbox{$\X  \times  \{-1,+1\}$},
the learner aims at finding a weigthing distribution $\posterior$ over $\Hcal$ that induces a low-error majority vote; in other words, minimizing the \emph{true risk} 
of the \mbox{$\posterior$-weighted} majority vote  \mbox{$\RP(\BQs)  =  \esp{(\xbf,\ybf)\sim D} \I \left[\BQ(\xbf)\ne \ybf\right]\,$}  under the \mbox{0-1-loss} is aimed for.
One way towards this goal is to implement the Empirical Risk Minimization (ERM) principle, that is to minimize the {\em empirical risk}  of the majority vote $\Risk_{S}(\BQ)$ estimated on the sample $S$.
Unfortunately, a well-known issue to learn such weights 
is that  the direct minimization of 
$\Risk_{S}(\BQ)$  is an ${\cal NP}$-hard problem.
To overcome this, we may use relaxations of the risk, 
look for estimators or bounds of the true risk that are simultaneously valid for all possible distributions~$\posterior$ on $\Hcalbf$, and try to minimize them. 
In the PAC-Bayesian theory \cite{Mcallester99a},
such an estimator is given by the  \emph{Gibbs risk} $\Risk(\GQ)$ of a \mbox{$\posterior$-weighted} majority vote which is simply the \mbox{$\posterior$-average} risk of the voters. Indeed, it is well known that the risk of the majority vote is bounded by twice its Gibbs risk:
\begin{equation}
\label{eq:gibbs}
\RP(\BQ)\ \leq\ 2\, \RP(\GQ)\ =\ 2\, \espdevant{\hbf\sim\posterior}\RP(\hbf)\,.
\end{equation}
With this relation,  PAC-Bayesian theory indirectly gives generalization bounds for \mbox{$\posterior$-weighted} majority votes.
Unfortunately, even if they tightly bound the true risk $\Risk_{D}(\GQ)$ in terms of its empirical counterpart $\Risk_{S}(\GQ)$, this tightness might not carry over to the bound on the  majority vote. 

Note that even if there exist situations for which Inequality~\eqref{eq:gibbs} is an equality, ensemble methods (especially when the voters are weak) build 
on the idea that the risk of the  majority vote might be way below the average of its voters' risk. 
Indeed, it is well known that voting can dramatically improve performances when the ``community'' of voters
tends to compensate the individual errors. The ``classical'' PAC-Bayesian framework of \citet{Mcallester99a} does not allow to 
evaluate whether or not this compensation occurs. To overcome this problem, \citet{Lacasse07} proposed 
not only to take into account the mean of the errors of the associated Gibbs $\Risk_{D}(\GQ)$, but also its variance. They proposed a new bound, called the \emph{\Cbound}, that replaces the loose factor of 2 in Inequality~\eqref{eq:gibbs}.
They also extended the PAC-Bayesian theory in such a way that both the mean and the variance of the Gibbs classifier can be estimated from the training data simultaneously for all $\posterior$'s. \citet{MinCq} have reformulated this approach in terms of the first and second statistical moment of the \emph{margin} realized by the \mbox{$\posterior$-weighted} majority vote, which pertains to known results in non-PAC-Bayesian frameworks, {\it e.g.,} \cite{Blanchard04,randomforests}; the margin of a $\posterior$-weighted majority vote 
on an example $(\xbf,\ybf) \in \X \times \Y$ being:
\begin{equation*}
\MQ(\xbf,\ybf) \ \eqdef \ \ybf \Esp_{\hbf\sim \posterior} \hbf(\xbf)\,.
\end{equation*}
In terms of this margin, the \Cbound is defined as follows.

\begin{theorem}[\Cbound of \citet{MinCq}]
	\label{theo:Cbound_bin}	\label{theo:cborne_binaire}
	For every distribution $\posterior$ on a set of voters  $\Hcalbf$, and for every distribution $D$ on $\X  \times  \Ybf$, 
	if  $\ \Esp_{(\xbf,\ybf)} \ybf \Esp_{\hbf\sim\posterior} \hbf(\xbf)   >  0$, then we have: 	
        \begin{equation*}
		\RP(\BQ)\ \leq\ 1- \frac{\displaystyle \left(\Esp_{(\xbf,\ybf)\sim D} \ybf \Esp_{\hbf\sim\posterior} \hbf(\xbf)\right)^2}{\displaystyle\Esp_{(\xbf,\ybf)\sim D} \left(\ybf \Esp_{\hbf\sim\posterior} \hbf(\xbf)\right)^2}\,.
	\end{equation*}
\end{theorem}
\begin{proof}First,  note that 
$$\RP(\BQs)  =  \Prob_{\ex\sim D}   \left(\MQ\ex \leq 0\right)\,,$$
 then 
	to upper-bound $\RP(\BQs)$ it suffices to upper-bound \mbox{$\Prob_{\ex\sim D}   \left(\MQ\ex \leq 0\right)$}.
	Making use of the Cantelli-Chebyshev inequality that states that for any random variable $Z$,
	$$ \forall a >  0, \ \Prob \left(Z\leq \Esp_{}\left[Z\right] - a\right) \leq \frac{\Var Z}{\Var Z + a^2}\,,$$
we get the desired result if \mbox{$Z   =   \MQs(\xbf,\ybf)$}, with {$a   =   \Esp_{(\xbf,\ybf)\sim D} \MQs(\xbf,\ybf)$} combined with the definition of the variance. The constraint \mbox{$\ \Esp_{(\xbf,\ybf)} \ybf \Esp_{\hbf\sim\posterior} \hbf(\xbf)   >  0$} comes from the fact that this inequality is valid for $a > 0$.
\end{proof}

\noindent The \Cbound  involves both the  $\posterior$-weighted majority vote confidence via {$\Esp_{(\xbf,\ybf)} (\ybf \Esp_{\hbf\sim\posterior} \hbf(\xbf))$} 
and the correlation between the voters via { $\Esp_{(\xbf,\ybf)} (\ybf \Esp_{\hbf\sim\posterior} h(\xbf))^2$}. It
is known to be very precise. Minimizing its empirical counterpart then appears as a natural solution for learning a distribution $\posterior$ leading to a well-performing  binary \mbox{$\posterior$-weighted} majority vote.
Moreover, this strategy is justified by a PAC-Bayesian generalization bound over the \Cbound (similar to Theorem~\ref{thm:pacbound} of this paper but restricted to the case where $\Ybf = \{-1,+1\}$), and has given the MinCq algorithm~\cite{MinCq}.

As announced earlier, we here intend to generalize the \Cbound theory to more complex outputs than  binary outputs. Our contributions first consist in generalizing---in Section~\ref{sec:cborne_generale}---this important result to a broader ensemble method setting, along with PAC-Bayesian generalization bounds.

\section{A General Setting for Majority Votes over a Set of Complex Output Voters}
\label{sec:cborne_generale}

In this section, we propose a general setting in which one can consider classification under $\posterior$-weighted majority votes. 
We present a general definition of the margin and propose a  \Cbound designed for majority vote-based ensemble methods when we want to combine complex output classifiers. We also discuss how to estimate this bound from a set $S$ 
of
$m$ examples drawn {\it i.i.d.} from $D$.
To do so, we derive a PAC-Bayesian theorem that bounds the \emph{true} risk  of the $\posterior$-weighted majority vote $\RP(\BQs)$  by using the empirical estimation of our new \Cbound on the sample~$S$.

\subsection{A General \Cbound for Complex Output Prediction}

Given some input space $\Xbf$ and a \emph{finite} output space $\Ybf$,
we suppose
that there exists
a feature map \mbox{$Y : \Ybf \to  H_\Ybf$}, where $H_\Ybf$ is a vector space such as a Hilbert space. 
For the sake of clarity, we suppose that all the vectors $Y(\Ybf)$ are unit-norm vectors; 
most of the following results  remain true without this assumption but have to be stated in a  more complicated form.
Let $\imY$ be the image of  $\Ybf$ under $Y(\cdot)$, and $\hullimY$ ($\subseteq H_\Ybf$) its convex hull.
We consider a (non-necessarily finite) set of \emph{voters}~$\Hcalbf \subseteq \left\{\hbf : \Xbf \to  \hullimY\right\}$. A voter that always outputs values in $\imY$ is called a classifier. For every probability distribution $\posterior$ on $\Hcalbf$, we define the \emph{\mbox{$\posterior$-weighted} majority vote classifier} $\BQs$ such that: 
\begin{align}
\forall \xbf\in\X,\quad \BQs(\xbf) =\argmin_{\cbf\in \Ybf}  \left\| Y(\cbf) -  \esp{\hbf\sim \posterior} \hbf(\xbf)\right\|^2 
   = \argmax_{\cbf\in \Ybf}   \left\langle Y(\cbf)\,,\,  \esp{\hbf\sim \posterior} \hbf(\xbf) - \frac12  Y(\cbf)\right\rangle.\label{eq:maj_s}
\end{align}
As in  the binary classification case, the learning objective 
is to find a distribution $\posterior$ that minimizes the {\em true risk} $\RP(\BQs)$
of the $\posterior$-weighted majority vote given by
\begin{align*}
\RP(\BQs)\ &=\ \Esp_{(\xbf,\ybf)\sim D}\, \I\left[\BQs(\xbf)\ne \ybf\right]\,.
\end{align*}
Inspired by the margin definition of~\citet{randomforests}, we propose the following generalization of the binary margin, which measures
the confidence of a prediction as the deviation between the voting weights received by the correct class
and 
the largest voting weight received by any incorrect class.
\begin{definition}\label{def:notre_marge}
	For any example $(\xbf,\ybf)$ and any distribution~$\rho$  on a set of voters $\Hcalbf$, we define the \emph{margin $\MQs(\xbf,\ybf)$ of the \mbox{$\posterior$-weighted} majority vote} on $(\xbf,\ybf)$  as:
	\begin{align}
	\MQs(\xbf,\ybf) 
	\nonumber
	& \eqdef \left\langle Y(\ybf)\,,\,\Esp_{\hbf\sim \posterior} \hbf(\xbf) -\tfrac12 Y(\ybf) \right\rangle  - \max_{\substack{\cbf\in\Ybf\\\cbf\neq \ybf}} \left\langle Y(\cbf)\,,\,\Esp_{\hbf\sim \posterior} \hbf(\xbf) -\tfrac12 Y(\cbf) \right\rangle\\
	& \label{eq:def_marge}
	\eqdef \left\langle \Esp_{\hbf\sim \posterior} \hbf(\xbf)\,,\, Y(\ybf) \right\rangle 
	- \max_{\substack{\cbf\in\Ybf\\\cbf\neq \ybf}} \left\langle \Esp_{\hbf\sim \posterior} \hbf(\xbf) \,,\,Y(\cbf)  \right\rangle\,.
	\end{align}
\end{definition}
The second form of the margin readily comes from simple computations combined with our assumption that $\big\|Y(\ybf)\big\| = 1,\;\forall\ybf\in \Ybf$.

With this definition at hand, it is obvious that the \mbox{$\posterior$-weighted} majority vote errs on $\exbf$ if and only if the margin realized on $\exbf$ is negative. Therefore,
\begin{align}
\label{eq:risk_marge}
\RP(\BQ)\ = \Prob_{(\xbf,\ybf)\sim D}\,  \left(\MQs(\xbf,\ybf)\leq 0\right)\,.
\end{align}
\begin{remark}
	\label{rk:setting_bin}
	We may retrieve the binary notion of majority vote from our general framework in various ways.
	For example, one may take    $Y :  \{-1,+1\} \to  \mathbb{R}$ with $Y(+1) = 1$ and $Y(-1) = -1$ considering each 
	binary voters as a voter. 
	Note that the definition of the margin, and, henceforth, that of the \Cbound, will differ because the binary definition of the margin is linear in  $\Esp_{\hbf\sim\posterior}  \hbf(\xbf)$ whereas our margin definition is quadratic in that variable. To fall back to the exact same \Cbound from our framework, the squared norm should be replaced by the norm itself in Definition~\ref{def:notre_marge}. 
	We however purposely choose to work with the square of the norm
	as it renders the calculations easier.
\end{remark}

Using the proof technique of Theorem~\ref{theo:cborne_binaire}, 
we arrive at a general \Cbound.
\begin{theorem}[General \Cbound]
	\label{theo:cborne_generale}
	For every probability distribution $\posterior$ over a set of voters~$\Hcalbf$ from $\X$ to $\hullimY$, and for every distribution $D$ on $\X \times \Ybf$,
	if \mbox{$\ \Esp_{(\xbf,\ybf)\sim D} \MQs(\xbf,\ybf)>0\,,$} then we have:
        \begin{align*}
		\RP(\BQs) \ \leq \ 1-\frac{\displaystyle \left(\Esp_{(\xbf,\ybf)\sim D}  \MQs(\xbf,\ybf)\right)^2}{\displaystyle \Esp_{(\xbf,\ybf)\sim D} \left( \MQs(\xbf,\ybf)\right)^2}\,.
        \end{align*}
\end{theorem}
\begin{proof} Thanks to Equation~\eqref{eq:risk_marge} the proof consists in bounding  \mbox{$\Prob_{\exbf}   \left(\MQ\exbf  \leq  0\right)$} with the Cantelli-Chebyshev inequality as done for Theorem~\ref{theo:cborne_binaire}.
\end{proof}
\begin{remark}[On the construction of the set of voters $\Hcalbf$]
	All our  results hold for both extreme cases of weak voters, as usual in ensemble methods, and that of more expressive/highly-performing voters.
	Typical instantiations of the former situation are encountered when making use of a kernel function $k:X \times  X\to \mathbb{R}$ that induces the set of voters $\Hcalbf=\left\{ \,k(\xbf,\cdot)\,Y(\ybf)\, |\, (\xbf,\ybf)\in S\right\}$; the situation also arises when a set of structured prediction functions learned with different hyperparameters are considered; as evidenced in Section~\ref{sec:expe} for the multiclass setting, 
	the weak voters may also be decision stumps or (more-or-less shallow) trees.
	Combining more expressive voters is a situation that may show up as a need to combine voters obtained from a primary mechanism. This is for instance the case in multiview learning~\cite{sun2013survey} when we want to combine models learned from several data descriptions---note that, the binary \Cbound has already shown its relevance in such a situation~\cite{sspr14}.
\end{remark}

\subsection{A PAC-Bayesian Theorem to Estimate the General \Cbound}
\label{sec:pacbayes}

In this section, we briefly discuss how to estimate the previous bound from a sample $S$ constituted by $m$ examples drawn {\it i.i.d.} from $D$. To reach this goal, we derive a PAC-Bayesian theorem that upper-bounds the \emph{true} risk  $\RP(\BQs)$ of the $\posterior$-weighted majority vote by using the empirical estimation of the \Cbound of Theorem~\ref{theo:cborne_generale} on the sample $S$.

\begin{theorem}\label{thm:pacbound}
	For any distribution $D$ on \mbox{$\X \times  \Ybf$}, for any set~$\Hcalbf$ of voters from $\X$ to $\hullimY$,  
	for any prior distribution $\prior$ on $\Hcalbf$ and any $\delta \in  (0,1]$, with a probability at least $1 - \delta$ over the choice of the $m$-sample $S \sim (D)^m$, for every posterior distribution
 $\posterior$ over $\Hcalbf$,  if \mbox{$\ \Esp_{(\xbf,\ybf)\sim D}   \MQs(\xbf,\ybf) > 0\,,$} we have:
	\begin{equation*}
		\RP(\BQ)  \,\, \leq \,\, 
		1-\frac{\Bigg( \, \max \left[0, \frac1m     \sum_{(\xbf,\ybf)\in S}     \MQs(\xbf,\ybf)   -  \sqrt{  \frac{2B}{m}\LB \KL(\posterior \| \prior)  +  \ln\frac{2\sqrt{m}}{\delta/2}\RB}\right] \, \Bigg)^2 }{ \min \left[1, \frac1m      \sum_{(\xbf,\ybf)\in S}       \left( \MQs(\xbf,\ybf)\right)^2     +  \sqrt{  \frac{2B^2}{m}  \LB 2\KL(\posterior || \prior)  +  \ln\frac{2\sqrt{m}}{\delta/2}\RB}\right] } \,,
	\end{equation*}
	where $B \in\, (0,2]$  bounds the absolute value of the margin $|\MQ(\xbf,\ybf)|$ for all $(\xbf,\ybf)$, 
	and \mbox{\small$\KL(\posterior\|\prior)  =  \esp{\hbf\sim \posterior}\ln\frac{\posterior(\hbf)}{\prior(\hbf)}$} is the Kullback-Leibler divergence between $\posterior$ and $\prior$. 
\end{theorem}
\begin{proof}
	First since \mbox{\small$\|Y(\ybf)\| = 1,\ \forall \ybf \in \Ybf$}, and \mbox{\small$\hbf(\xbf) \in \hullimY,\ \forall \xbf \in \X$}, then
	\mbox{\small$\left\langle \Esp_{\hbf\sim \posterior} \hbf(\xbf) \,,\,Y(\cbf)  \right\rangle$} takes its value between $-1$ and $+1$. It follows from Equation~\eqref{eq:def_marge} that $B  =  2$ is always an upper bound of $|\MQ(\xbf,\ybf)|$.
	The bound is obtained by deriving a PAC-Bayesian lower bound on \mbox{\small$\Esp_{(\xbf, \ybf) \sim D}  \MQs(\xbf, \ybf)$} and a PAC-Bayesian upper bound on \mbox{\small$\Esp_{(\xbf, \ybf) \sim D} \left(\MQs(\xbf, \ybf)\right)^2$}. We then use the union bound argument to make these two bounds simultaneously valid, and the result follows from  Theorem~\ref{theo:cborne_generale}. These two bounds and their respective proof are provided in Appendix~\ref{sec:supp}, as Theorems~\ref{thm:borne1} and~\ref{thm:borne2}.
\end{proof}
	
Unlike to classical PAC-Bayesian bounds and especially those provided for structured output prediction by \citet{mcallester2009generalization}, our theorem has the advantage to directly upper-bound the risk of the $\posterior$-weighted majority vote thanks to the \Cbound of Theorem~\ref{theo:cborne_generale}. Moreover, it allows us to deal with either the general notion of margin, or margin's surrogate as illustrated in the following.

\subsection{A Surrogate for the Margin}
\label{sec:omega_borne}
The general notion of margin can be hard to exploit in general
because of its second term relying on a $\max$. 
We propose  to define a simpler surrogate of the margin, by replacing the second term in Equation~\eqref{eq:def_marge} by a threshold $\omega \in  [0,1]$.
\begin{definition}[The $\omega$-margin]
	\label{def:omega_marge_generique}
	For any  example
	$(\xbf,\ybf) \in\X \times \Y$ and for any  distribution $\posterior$ on $\Hcalbf$, we define the \emph{\mbox{$\omega$-\emph{margin}} $\OQ(\xbf, \ybf)$ of the \mbox{$\posterior$-weighted} majority vote} realized on $\exbf$ as\\
	\begin{align*}
		\OQ(\xbf, \ybf)\ \eqdef\ \left\langle \Esp_{\hbf\sim \posterior} \hbf(\xbf)\,,\,\Y(\ybf) \right\rangle\, - \, \omega\,.
		\end{align*}
\end{definition}
\noindent{}Trivially, the \mbox{$\omega$-margin} always upper-bounds the margin when $\omega  =  0$. 
Moreover, since  \mbox{\small$\forall Y(\ybf)\in\ima\Ybf, \|Y(\ybf)\| = 1$}, and \mbox{\small$\forall \xbf \in  \X,\ \Esp_{\hbf\sim \posterior}   \hbf(\xbf)  \in  \hullimY$}, then
the  \mbox{$\omega$-margin} always lower-bounds the margin when $\omega  =  1$. We will see that in the multiclass setting it is also the case for $\omega  =  \frac12$.
When the \mbox{$\omega$-margin} lower-bounds the margin, we can replace it in the \Cbound in the following way:
\begin{equation}
\label{eq:la_C_omega_pas_tjrs_borne}
{\cal C}(\OQ)\ \eqdef \ 1-\frac{\displaystyle\left(\Esp_{(\xbf,\ybf)\sim D}  \OQ\exbf\right)^2}{\displaystyle\Esp_{(\xbf,\ybf)\sim D} \left( \OQ\exbf\right)^2}\,.
\end{equation}
Indeed, in this situation we have:
$$\displaystyle \RP(\BQs)\ =\     \Pr_{(\xbf,\ybf)\sim D}    \left(\MQ(\xbf,\ybf)\leq 0\right)\ \leq\     \Pr_{(\xbf,\ybf)\sim D}    \left(\OQ(\xbf,\ybf)\leq 0\right).$$
Therefore, the proof process of Theorem~\ref{theo:cborne_generale} applies if \mbox{\small $\displaystyle \Esp_{(\xbf,\ybf)\sim D}\OQ\exbf>0$.}

Note that, even for values of $\omega$ for which ${\cal C}(\OQ)$ does not give rise to a valid  upper bound of $\RP(\BQs)$, it is still interesting to calculate it as there are some values that can be very good estimators of $\RP(\BQs)$ simultaneously for many different values of $\rho$. If so, one can capture the behavior of the ensemble method this way. We provide some evidence about this in Section~\ref{sec:expe}.

We now theoretically and empirically illustrate these results by addressing the multiclass classification issue from our previous general \Cbound perspective.

\section{Specializations of the \Cbound to Multiclass Prediction}
\label{sec:cborne_multi}
\label{sec:multiclasse}

\subsection{From Multiclass Margins to \Cbounds} 
\begin{figure}[t]
	\centering
        \includegraphics[width=0.6\textwidth]{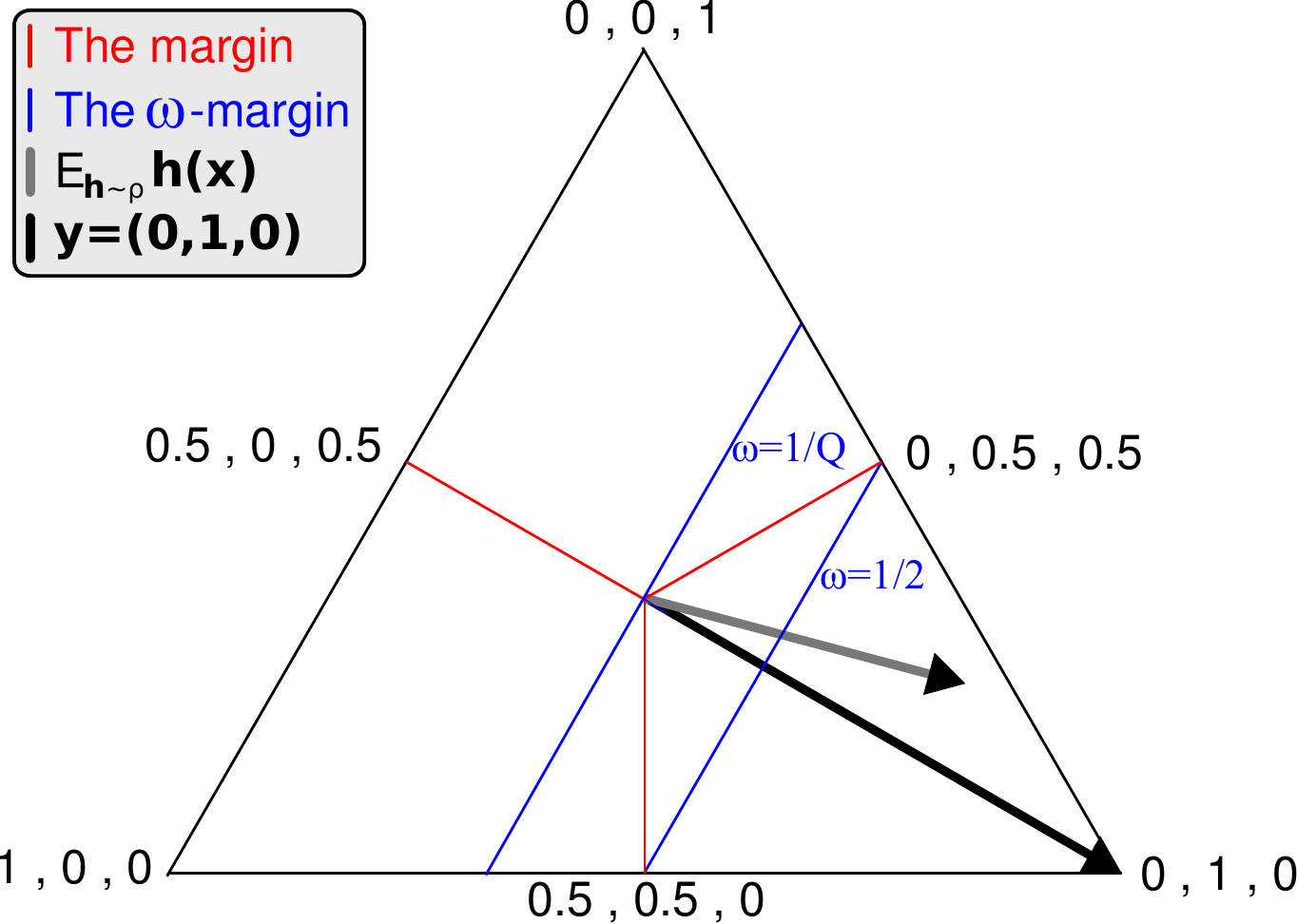}
	\caption{\label{fig:multiclass}
	Representation of the multiclass margins and the vote applied on $(\xb,\ybf)$ in the barycentric coordinate system defined by $\hullimY$ when \mbox{$\Ybf   =   \{1,2,3\}$} and the true $\ybf$
	is $2$, {\it i.e.}, $(0,1,0)^\top$. We have \mbox{$Y(1)  =  (1,0,0)^\top$}, \mbox{$Y(2)  =  (0,1,0)^\top$}, and \mbox{$Y(3)  =  (0,0,1)^\top$}.  Each line is the decision boundary of a margin: the hyperplane where lies each example with a margin equals to $0$. A vote correctly classifies an example if it lies on the same side of the hyperplane than the correct class.}
\end{figure}
The input space at hand still is \mbox{$\X$}, but the output space \mbox{$\Ybf  =  \{1,\ldots,Q\}$} is now made of a finite number of classes (or categories) $Q \geq  2$. 
We define the output feature map $Y(\cdot)$ such that the image of $\Ybf$ is \mbox{$\imY = \{0,1\}^{\Q}$}.
More precisely, the  image of a class \mbox{$\ybf \in \Ybf$} under $Y(\cdot)$ is the canonical \mbox{$Q$-dimensional} vector $(0, \ldots, 1,\ldots 0)^{\top}$ whose only nonzero entry is a $1$ at its $\ybf$-th position.
The set~$\Hcalbf$ is a set of multiclass voters $\hbf$ from $\X$ to $\hullimY$.
We recall that given a prior distribution $\prior$ over $\Hcalbf$ and an {\it i.i.d.} $m$-sample $S$ (drawn from $D$), the goal of the PAC-Bayesian theory is to estimate the prediction ability of the \mbox{$\posterior$-weighted}
majority vote $\BQs(\cdot)$ of Equation~\eqref{eq:maj_s}. In this multiclass setting, since for each  $\cbf \in \Ybf$ only the \mbox{$\cbf$-th} coordinate of $Y(\cbf)$ equals to $1$,  the definitions of the majority vote classifier and the margin  can respectively be rewritten as: 
\begin{align*}
\BQs(\xbf) \ &=\ \argmax_{\cbf\in\Ybf}  \gc{\xbf}{\cbf}\,,\\
text{and}\ \ \
\MQ(\xbf,\ybf)\   &= \   \gc{\xbf}{\ybf} -    \max_{\cbf\in\Ybf,\cbf\neq \ybf}\, \gc{\xbf}{\cbf}\,,
\end{align*}
where $\hbf_\cbf(\xbf)$ is the $\cbf$-th coordinate of $\hbf(\xbf)$.
The following theorem relates $\RP(\BQs)$ 
and the \mbox{$\omega$-margin} associated to the distribution
$\posterior$ over $\Hcalbf$.
\begin{theorem}
	\label{theo:link}
	Let $Q\geq 2$ be the number of classes. For every distribution $D$ over \mbox{$\X \times  \Ybf$} and for every distribution $\posterior$ over a set of multiclass voters $\Hcalbf$,  we have:
	\begin{equation*}  \Prob_{\exbf\sim D}  \left(\OQQ\exbf  \leq  0\right)\  \leq\   \RP(\BQs) \   \leq \  \Prob_{\exbf\sim D}  \left(\OQtwos\exbf  \leq  0\right) .
	\end{equation*}
\end{theorem}
\begin{proof}
	First, let us prove the left-hand side inequality.
	\begin{small}
		\allowdisplaybreaks[1]
		\begin{align*}
		\RP(\BQs)\  &=\ \Prob_{(\xbf,\ybf)\sim D}  \left(\MQ\exbf \leq 0\right)\\
		&= \ \Prob_{(\xbf,\ybf)\sim D}  \left( \gc{\xbf}{\ybf} \leq \max_{\cbf\in\Ybf, \cbf\ne \ybf} \gc{\xbf}{\cbf}  \right)\\
		&\geq  \  \Prob_{(\xbf,\ybf)\sim D} \left( \gc{\xbf}{\ybf} \leq \frac{1}{Q-1}\sum_{\cbf=1, \cbf\neq \ybf}^{\Q} \,\gc{\xbf}{\cbf}  \right)\\
		&= \  \Prob_{(\xbf,\ybf)\sim D}  \left( \gc{\xbf}{\ybf} \leq \frac{1}{Q - 1}\left[1 - \gc{\xbf}{\ybf}\right]  \right)\\
		&= \  \Prob_{(\xbf,\ybf)\sim D} \left(  \gc{\xbf}{\ybf}  - \frac{1}{Q} \leq 0\right)\\
                & = \  \Prob_{(\xbf,\ybf)\sim D}  \left( \OQQ\exbf \leq 0\right)\,.
		\end{align*}%
	\end{small}%
	The right-hand side inequality is  verified by observing that $\BQs(\cdot)$ necessarily makes a correct prediction if the weight $\gc{\xbf}{\ybf}$  given to the correct $\ybf$ is higher than $\frac12$.
\end{proof}
Consequently, as illustrated in Figure~\ref{fig:multiclass}, the \mbox{$\omega$-margin} of the points that lie between the \mbox{$\frac{1}{\Q}$-margin} and the \mbox{$\frac12$-margin} can be negative or positive according to 
$\omega$. 
We thus have the following bound.
\begin{cor}[$\omega$-margin multiclass \Cbound]
	\label{cor:marge_omega}
	For every probability distribution $\posterior$ on a set of multiclass voters $\Hcalbf$, and for every distribution $D$ on \mbox{$\X \times  \Ybf$}, if  \mbox{$\,\esp{\exbf\sim D} \OQtwo\exbf > 0$}, then we have:
\begin{align*}
		\RP(\BQs) \  \leq\ {\cal C}(\OQtwo)\ =\  1 - \frac{\left(\displaystyle\esp{\exbf\sim D} \OQtwo\exbf\right)^2}{\displaystyle\esp{\ex\sim D} \left(\OQtwo\exbf\right)^2}\,,
\end{align*}
	where ${\cal C}(\cdot)$ is the function involved in the \mbox{$\omega$-margin}-based \Cbound  (Equation~\eqref{eq:la_C_omega_pas_tjrs_borne}).
\end{cor}

The region of indecision when \mbox{$\omega \in${\footnotesize$\,[\frac{1}{\Q},\frac12]$}} implies there is possibly some value of $\omega$ to be chosen carefully to provide a good estimator of the true margin. If this is so,  we can then consider to  make use of  ${\cal C}(\OQ)$ for that particular value of $\omega$ to improve the analysis of the majority vote's behavior. Obviously, in such a situation, ${\cal C}(\OQ)$ is no longer a bound on $\RP(\BQs)$.
However, due to the linearity of the \mbox{$\omega$-margin}, this could open the way to a generalization of the algorithm MinCq~\cite{MinCq} to the multiclass setting.

\subsection{Experimental Evaluation of the Bounds}
\label{sec:expe}

\begin{figure}[t]
	\centering
        \includegraphics[width=\linewidth]{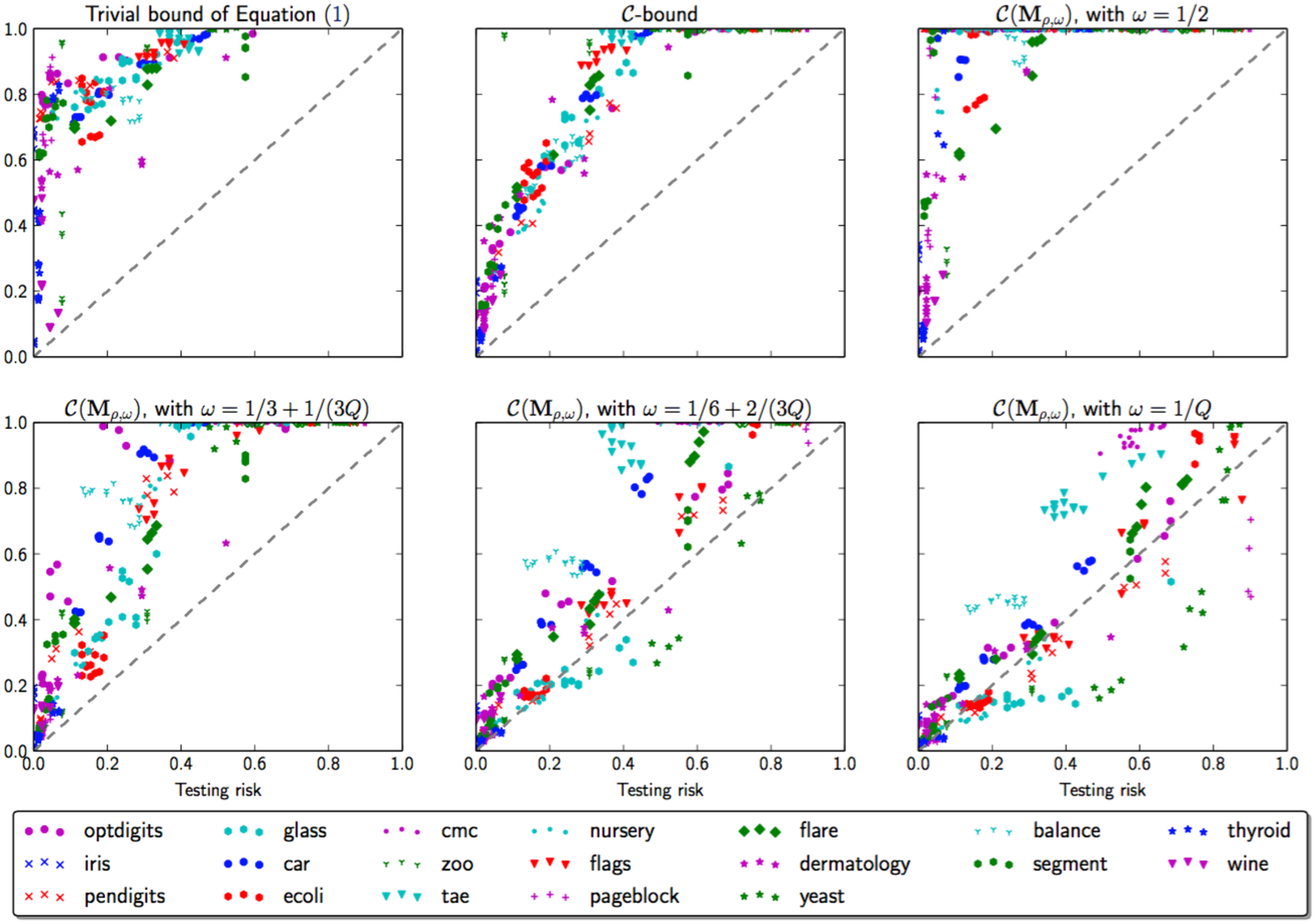}
	\caption{Comparison of the \emph{true} risk of the \mbox{$\posterior$-weighted} majority with: the trivial bound of Equation~\eqref{eq:gibbs},  the \Cbound, and ${\cal C}(\OQ)$ for various values of $\omega$. All the values was calculated on a testing set disjoint from the one  used to learn~$\posterior$.  
	}
	\label{fig:cbound_correlations}
\end{figure}

\begin{table}[t]\small
	\caption{\label{tab:pearson}Pearson correlations with the testing risk of the majority vote.}
	\begin{center}
		\rowcolors{5}{}{black!10}
		\begin{tabular}{lc}
			\toprule
			Quantity (evaluated on set $T$) & Pearson correlation with $\textbf{R}_T(\BQ)$\\
			\cmidrule(l r){1-1} \cmidrule(l r){2-2}
			Trivial bound of Equation~\eqref{eq:gibbs}  &  $0.7525$ \\
			${\cal C}(\MQ)$, the Multiclass (the \Cbound of Theorem~\ref{theo:cborne_generale}) & $0.8791$\\
			${\cal C}(\OQ)$  with $\omega=1/2$ (The bound of Corollary~\ref{cor:marge_omega}) & $0.5688$\\
			${\cal C}(\OQ)$ with $\omega=1/3 + 1/(3Q)$ (not a bound)& $0.8838$\\
			${\cal C}(\OQ)$ with $\omega=1/6 + 2/(3Q)$ (not a bound)& $\mathbf{0.9090}$\\
			${\cal C}(\OQ)$ with $\omega=1/Q$ (not a bound)& $0.8741$\\
			\bottomrule
		\end{tabular}
	\end{center}
\end{table}

The binary \Cbound is known to be well-suited to characterize the behavior of the risk of the \mbox{$\posterior$-weighted} majority vote, as their respective values are correlated~\citep{Lacasse07}. We extend this analysis by empirically evaluating the behavior of the multiclass \Cbounds introduced above on natural data. 
We generate multiclass \mbox{$\posterior$-weighted} majority votes by running a multiclass version of AdaBoost~\cite{fs-97}---known as AdaBoost-SAMME\footnote{We  use of the implementation 
	provided in the Scikit-Learn Python library~\cite{scikit-learn}.}~\cite{zhu2009multi}---on multiclass datasets from the UCI dataset repository~\cite{uci-98}. We split each dataset in two halves: a training set $S$ and a testing set $T$. We train the algorithm on $S$, using $100$, $250$, $500$ and $1,000$ decision trees of depth $2$, $3$, $4$ and $5$ as base voters, for a total of sixteen majority votes per dataset. The reported values are all computed on the testing set. Figure~\ref{fig:cbound_correlations} shows how the values of different upper bounds relate with the risk of the  majority vote, and how the choice of $\omega$ for various values of ${\cal C}(\OQ)$  (it is not an upper bound for $\omega  <  \frac12$) affects the correlation with the risk. We finalize this correlation study by reporting, in Table~\ref{tab:pearson}, the Pearson product-moment correlation coefficients for all computed values.

As pointed out in the paper, we notice from Figure~\ref{fig:cbound_correlations} and Table~\ref{tab:pearson} that 
for some values $\omega$, the values of ${\cal C}(\OQ)$ 
are very correlated with the risk of the majority vote. 
Unfortunately, the only one that is an upper bound of the latter (\mbox{$\omega=\frac12$}) does not show the same predictive power.
Thus, these results also gives some empirical evidence that a wise choice of $\omega$ can improve the correlation between the \Cbound based on the \mbox{$\omega$-margin} and  the risk of the vote.

These experiments confirm the usefulness of the \Cbounds based on a notion of margin to upper-bound the true risk of the {$\posterior$-weighted} majority vote. Taking into account the first and second statistical moments of such margins seems effectively very informative. This property is interesting in an algorithmic point of view: one could derive a multiclass optimization algorithm generalizing the algorithm MinCq \citep{MinCq} by minimizing ${\cal C}(\OQ)$ where
$\omega$ could be
 a hyperparameter to tune by cross-validation.

\section{Specializations of the \Cbound to Multilabel Prediction}
\label{sec:cborne_multilabel}

In this section, we instantiate the general \Cbound approach to multilabel classification.
We stand in the following setting,
where the space of possible labels is \mbox{$\{1,\ldots,Q\}$} with a finite number of classes $Q  \geq   2$, but we consider the \emph{multilabel} output space  \mbox{$\Ybf  =  \{0, 1\}^{\Q}$} that contains vectors  \mbox{$\ybf = (y_1,\ldots,y_{\Q})^\top$}. In other words we consider multiple binary labels.
Given an example  \mbox{$(\xbf,\ybf)  \in  \X  \times  \Ybf$}, the output vector  \mbox{$\ybf$}  is then defined as follows:
$$ \forall j \in  \{1,\ldots,Q\},\quad  y_j = \left\{\begin{array}{l}1\text{ if $\xbf$ is labeled with $j$}\\0 \text{ otherwise}.\end{array}\right.$$
In this specific case, we define the output feature map $Y(\cdot)$  such that the image of $\Ybf$ is $\imY=${\footnotesize$\left\{\frac{-1}{\sqrt{\Q}},\frac{1}{\sqrt{\Q}}\right\}^{\Q}$}, and: 
$$\displaystyle \forall j \in  \{1,\ldots,Q\},\quad Y_j(\ybf)= \left\{\begin{array}{l}\frac{1}{\sqrt{\Q}}\text{ if $y_j = 1$ {\small ($\xbf$ is labeled with $j$)}}\\\frac{-1}{\sqrt{\Q}}\text{ otherwise},\end{array}\right.$$
where $Y_j(\ybf)$ is the $j$-th coordinate of $Y(\ybf)$.
According to this definition, we have that:  $\forall\ybf  \in \Ybf,\ \|Y(\ybf)\|=1$.
The set $\Hcalbf$ is made of \emph{multilabel voters} $\hbf : \X \mapsto \hullimY$.  
In the light of  the feature output map $Y(\cdot)$, the definition of the majority vote classifier and the margin can respectively be rewritten as:
\begin{align*}
\nonumber \BQs(\xbf)  \ 
&= \ \argmax_{\cbf \in \Ybf}\,   \sum_{j=1}^{\Q}\, \gc{\xbf}{j}Y_j(\ybf)  \,,\\[-2mm]
\text{and}\quad \MQs(\xbf,\ybf) \ &=  \ \sum_{j=1}^{\Q}\gc{\xbf}{j}Y_j(\ybf) - \max_{\substack{\cbf\in\Ybf\\\cbf\neq \ybf}}\,  \sum_{j=1}^{\Q}\gc{\xbf}{j}Y_j(\cbf)\,,
\end{align*}
where $\hbf_{j}$ is the $j$-th coordinate of $\hbf(\xbf)$.

\begin{figure}[t]
	\centering
\includegraphics[width=0.6\textwidth]{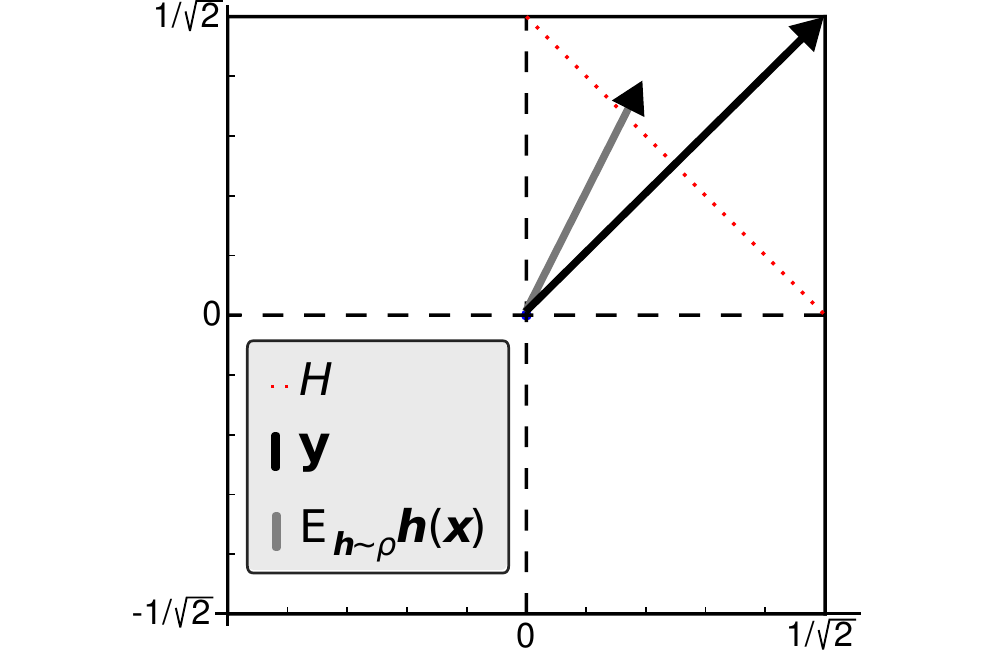}
\caption{ \label{fig:multilabel}Graphical representation of the \mbox{$\tfrac{\Q- 1}{\Q}$-margin} and the vote applied on an example $(\xb,\ybf)$ for multilabel classification 
when $Q  =  2$ and the true $\ybf$ is $(\fracq,\fracq)^\top $.  The angles of the cube corresponds to the different multilabels, that are: 
	$Y(\Ybf) = \{(\mfracq,\mfracq)^\top  ,(\fracq,\mfracq)^\top  ,(\mfracq,\fracq)^\top  ,(\fracq,\fracq)^\top \}$. Each line represents the decision boundary of a margin: the hyperplane where lies each example with a margin equals to $0$. A vote correctly classifies an example if it lies on the same side of the hyperplane than the correct class.}
\end{figure}

The next theorem relates the risk of $\BQ(\cdot)$ and the \mbox{$\omega$-margin} associated to $\posterior$.
\begin{theorem}
	\label{theo:omega_risk_multilabel}
	Let $Q\geq 2$ be the number of  labels. 
	For every distribution $D$ over $\X \times\Ybf$ and for every distribution $\posterior$ over a set of multilabel voters $\Hcalbf$, we have:
	$$
	\RP(\BQs)   \leq  \Prob_{\exbf\sim D}\left(\OQbof\exbf \leq 0\right)\,.
	$$
\end{theorem}
\begin{proof}
	We have to show: $$
			\Prob_{\exbf\sim D}\left(\MQ\exbf \leq 0\right) \leq \Prob_{\exbf\sim D}\left(\OQbof\exbf \leq 0\right)\,.$$
	To do so we will prove that: $$\OQbof\exbf> 0\Longrightarrow \MQ\exbf> 0\,.$$

	Recall that $\hullimY$ is a hypercube whose vertices are exactly the $Y(\cbf)$'s with $\cbf  \in   \Ybf$. Given a vertex $Y(\ybf)$, denote $H_\ybf$ the hyperplane who passes through all the points $Y^{(j)}(\ybf)$, where $Y^{(j)}(\ybf)$ is the point of the hypercube that has exactly the same coordinates as $Y(\ybf)$, except the $j^{\mbox{\tiny th}}$ that has been put to $0$.
	
	Now, consider the region $R_\ybf$ of the hypercube $\hullimY$ that consists of all the points that correspond to \mbox{\small$\OQbof\exbf >  0$}, that is, the points that are on the same side of hyperplane $H_\ybf$ than $Y(\ybf)$. Clearly, for any $Q\geq 2$, 
	$Y(\ybf)$ is strictly closer to the point \mbox{\small$\Esp_{\hbf\sim\posterior}\hbf(\xbf)$} than any other $Y(\cbf)$'s if the vector \mbox{\small$\Esp_{\hbf\sim\posterior}\hbf(\xbf)$} lies in $R_\ybf$. This in turn implies that the margin $\MQ\exbf$ is strictly positive. Figure~\ref{fig:multilabel} shows an example with $Q=2$ and \mbox{\small$Y(\ybf) = (\frac{1}{\sqrt{2}},\frac{1}{\sqrt{2}})$}, where $H_\ybf$ is represented by a red dotted line, and $R_\ybf$ is the region delimited by the top-right corner and $H_\ybf$.
	
	To finish the proof, we have to show that  $R_\ybf$ is exactly the region where \mbox{\small$\OQbof\exbf >  0$}. Equivalently, we must show that the intersection of $H_\ybf$ and the hypercube $\hullimY$ is exactly the points for which \mbox{\small$\OQbof\exbf =  0$}, \emph{i.e.,} the vectors \mbox{\small$\Esp_{\hbf\sim\posterior}\hbf(\xbf)$} for which \mbox{\small$\langle \Esp_{\hbf\sim\posterior}\hbf(\xbf)\,,\, \ybf \rangle  - \frac{\Q- 1}{\Q}  =  0$}. We know from basic linear algebra  that the points $P$ that lie on hyperplane $H_\ybf$ must satisfy the following equation: $(P-P_0)\cdot N = 0$, where N is the normal of the hyperplane and $P_0$ is any point in $P$. It is easy to see that $Y(\ybf)$ is the normal of $H$ and that we can take  \mbox{\small$P_0 = Y^{(1)}(\ybf)$}. Hence, the equation becomes \mbox{\small$(P-Y^{(1)}(\ybf))\cdot Y(\ybf) = 0$}.
	
	Since all coordinates of $Y(\ybf)$ are either \mbox{\small$\frac{1}{\sqrt{\Q}}$} or  \mbox{\small$\frac{-1}{\sqrt{\Q}}$}, and all coordinates of $Y^{(1)}(\ybf)$ are the same as the ones of $Y(\ybf)$ except the first one being $0$ in $Y^{(1)}(\ybf)$, we have that \mbox{\small $Y^{(1)}(\ybf) \cdot Y(\ybf)  =  \frac{\Q- 1}{\Q}$}. The result then follows from
	\begin{equation*}
		(P-Y^{(1)}(\ybf))\cdot Y(\ybf) \ =\ P\cdot Y(\ybf)  -  Y^{(1)}(\ybf) \cdot Y(\ybf)\ =\  \left\langle P\,,\, Y(\ybf) \right\rangle  - \frac{Q- 1}{Q}\,.
	\end{equation*}
\end{proof}

Finally, according to the same arguments as in Corollary~\ref{cor:marge_omega}, we have:
\begin{cor}[$\omega$-margin multilabel \Cbound]
	\label{cor:marge_omega2_multilabel}
	For every probability distribution $\posterior$ on a set of multilabel voters $\Hcalbf$, for every distribution $D$ on \mbox{$\X  \times   \Ybf$},  if  \mbox{\small$\ \esp{\exbf\sim D} \OQbofs\exbf > 0$}, we have:
$$
		\RP(\BQ)\leq\ {\cal C}(\OQbof) \ =\ 1 - \frac{\left(\displaystyle\esp{\exbf\sim D} \OQbofs\exbf\right)^2}{\displaystyle\esp{\ex\sim D} \left(\OQbofs\exbf\right)^2}\,.$$
\end{cor}

\section{Conclusion}
\label{sec:conclusion}
\label{sec:conclu}

In PAC-Bayesian binary classification, it is well-known that the \Cbound offers a tight bound over the risk of the \mbox{$\posterior$-weighted} majority vote by taking into account the first two statistical moments of its margin. Moreover, from a practical standpoint, minimizing the \Cbound leads to a well-performing algorithm called MinCq \cite{MinCq}. This paper fills the gap between this binary classification theory and more complex tasks by generalizing the \Cbound for majority vote over complex output voters, and by proposing a new surrogate of the margin easier to manipulate. Note that, as future work, we would like to study how tuning this surrogate margin would result in a precise estimation of the risk. In order to justify the empirical estimation of the \Cbound from a sample, we provide a PAC-Bayesian generalization bound.
Moreover, we show how to specialize our  result to multiclass  and multilabel
classification.
Concretely, we think that the theoretical \Cbounds provided here are a first step towards developing ensemble methods to learn \mbox{$\posterior$-weighted} majority vote  for complex outputs through the minimization of a \Cbound, or of a surrogate of it.
A first solution for deriving such a method could be to study the general weak learning conditions necessary and sufficient to define an ensemble of structured output voters, as done by Mukherjee and Shapire~\cite{mukherjee2013theory} for multiclass boosting.

\appendix
\section*{Appendix}

\section{The Bounds Requiered to Prove Theorem~\ref{thm:pacbound}}
\label{sec:supp}

\begin{theorem}\label{thm:borne1}
	For any distribution $D$ on \mbox{$\X\!\times\! \Ybf$}, for any set~$\Hcalbf$ of voters from $\X$ to $\hullimY$,  
	for any prior distribution $\prior$ on $\Hcalbf$ and any $\delta\!\in\! (0,1]$, with a probability at least $1\!-\!\delta$ over the choice of the $m$-sample $S\!\sim\!(D)^m$, for every posterior distribution $\posterior$ over $\Hcalbf$ we have :
	\begin{equation*}
		\Esp_{(\xbf,\ybf)\in D} \MQs(\xbf,\ybf)  \,\, \geq \,\, 
		\frac1m     \sum_{(\xbf,\ybf)\in S}     \MQs(\xbf,\ybf)   -  \sqrt{  \frac{2B}{m}\LB \KL(\posterior \| \prior)  +  \ln\frac{2\sqrt{m}}{\delta}\RB} \,,
	\end{equation*}
	where $B\!\in\, (0,2]$  bounds the absolute value of the margin $|\MQ(\xbf,\ybf)|$ for all $(\xbf,\ybf)$, 
	and \mbox{\small$\KL(\posterior\|\prior) \!=\! \esp{\hbf\sim \posterior}\ln\frac{\posterior(\hbf)}{\prior(\hbf)}$} is the Kullback-Leibler divergence between $\posterior$ and $\prior$. 
\end{theorem}

\begin{proof}The following proof shows how to obtain the lower bound on the first moment of $\MQs\exbf$, and uses the same notions as the classical PAC-Bayesian proofs.\footnote{The reader can refer to~\citep{gllm-09,seeger2003pac,catoni2007pac,mcallester2003simplified} for examples of classical PAC-Bayesian analyses.}

Given a distribution $D'$ on \mbox{$\X  \times  \Ybf$},  for any distribution $\posterior'$ over $\Hcalbf$, we can rewrite \mbox{$\Esp_{(\xbf, \ybf) \sim D'} \MQprimes(\xbf, \ybf)$} as an expectation over $\posterior'$. We denote $\Mbf_\hbf^{D'}$ the random variable such that \mbox{$\Esp_{\hbf\sim \posterior'} \Mbf_\hbf^{D'} = \Esp_{(\xbf, \ybf) \sim D'} \MQprimes(\xbf, \ybf)$}. 
	
	First, we have that $\Esp_{\hbf \sim \prior} \exp\left[\frac{m}{2B}\left(\Mbf_\hbf^S - \Mbf_\hbf^D\right)^2\right]$ is a non-negative random variable. Applying Markov's inequality yields that with probability at least $1 - \delta$ over the choice of $S \sim (D)^m$, we have:
	\begin{equation}
	\label{eq:inequality2}
	\Esp_{\hbf \sim \prior} \exp\left[\frac{m}{2B}\left(\Mbf_\hbf^S - \Mbf_\hbf^D\right)^2\right]\ \leq\ \frac{1}{\delta} \Esp_{S \sim D^m} \Esp_{\hbf \sim \prior} \exp\left[\frac{m}{2B}\left(\Mbf_\hbf^S - \Mbf_\hbf^D\right)^2\right]\,.
	\end{equation}
	We upper-bound the right-hand side of the inequality:
		\begin{align}
\!\!	\Esp_{S \sim D^m}\, \Esp_{\hbf \sim \prior}\, &\exp\!\left[\frac{m}{2B}\left(\Mbf_\hbf^S - \Mbf_\hbf^D\right)^2\right]
		\label{line:expectations2}  = \Esp_{\hbf \sim \prior}\, \Esp_{S \sim D^m} \exp\!\left[\frac{m}{2B}\left(\Mbf_\hbf^S - \Mbf_\hbf^D\right)^2\right]\\
		\nonumber & =\ \Esp_{\hbf \sim \prior} \,\Esp_{S \sim D^m} \exp\!\left[m \, 2\left({ \frac{1}{2}\Big(1 -  \frac1B\Mbf_\hbf^S\Big) - \frac{1}{2}\Big(1 - \frac1B\Mbf_\hbf^D\Big)}\right)^2\right]\\
		\label{line:pinsker2} & \leq\ \Esp_{\hbf \sim \prior}\, \Esp_{S \sim D^m} \exp\!\left[m \, \kl\left({\frac{1}{2}\Big(1 -  \frac{\Mbf_\hbf^S}{B}\Big)}\middle\|{\frac{1}{2}\Big(1 - \frac{\Mbf_\hbf^D}{B}\Big)}\right)\right]\\
		\label{line:maurer2} & \leq\ \Esp_{\hbf \sim \prior} 2 \sqrt{m}\ =\ 2\sqrt{m}\,.
		\end{align}%
	Line~\eqref{line:expectations2} comes from the fact that the distribution $\prior$ is defined a priori. Since $B$ is an upper bound of the possible absolute values of the margin, both $\frac{1}{2}\Big(1 -  \frac{\Mbf_\hbf^S}{B}\Big)$ and $\frac{1}{2}\Big(1 -  \frac{\Mbf_\hbf^D}{B}\Big)$ are between $0$ and $1$. Thus  Line~\eqref{line:pinsker2} is an application of Pinsker's inequality $2(q-p)^2 \leq \kl(q\|p) = q\ln\frac{q}{p} +(1-q)\ln\frac{1-q}{1-p}$. Finally, Line~\eqref{line:maurer2} is an application of \mbox{\citet{m-04} (Theorem 5.)}.
	
	By applying this upper bound in Inequality~\eqref{eq:inequality2} and by taking the logarithm on each side, with probability at least $1 - \delta$ over the choice of $S \sim D^m$, we have:
	\begin{equation*}
		\ln \left(\Esp_{\hbf \sim \prior} \exp\left[\frac{m}{2B}\left(\Mbf_\hbf^S - \Mbf_\hbf^D\right)^2\right]\right)\ \leq\ \ln \left( \frac{2\sqrt{m}}{\delta} \right).
	\end{equation*}
	Now, by applying the change of measure inequality proposed by \mbox{\citet{seldin-tishby-10} (Lemma 4.)} with \mbox{$\phi(\hbf) = \frac{m}{2B}\left(\Mbf_\hbf^S - \Mbf_\hbf^D\right)^2$}, and by using Jensen's inequality exploiting the convexity of $\phi(\hbf)$, we obtain that for all distributions $\posterior$ on $\Hcalbf$:
	\begin{align*}
	 \ln \left(\Esp_{\hbf \sim \prior} \exp\left[\frac{m}{2B}\left(\Mbf_\hbf^S - \Mbf_\hbf^D\right)^2\right]\right) \ &\geq\ \Esp_{\hbf\sim \posterior} \frac{m}{2B}\left(\Mbf_\hbf^S - \Mbf_\hbf^D\right)^2 - \KL(\posterior\|\prior)\\
	&\geq\ \frac{m}{2B}\left(\Esp_{\hbf\sim \posterior} \Mbf_\hbf^S - \Esp_{\hbf\sim \posterior} \Mbf_\hbf^D\right)^2 - \KL(\posterior\|\prior)\,.
	\end{align*}	
	From all what precedes, we have that with probability at least $1 - \delta$ over the choice of $S \sim (D)^m$, for every posterior distribution $\posterior$ on $\Hcalbf$, we have:
	\begin{align*}
	\frac{m}{2B}\left(\Esp_{(\xbf, \ybf) \sim S}\MQs(\xbf, \ybf) -    \Esp_{(\xbf, \ybf) \sim D}\MQs(\xbf, \ybf)\right)^2     - \KL(\posterior\|\prior)\ \leq\ \ln \left( \frac{2\sqrt{m}}{\delta} \right).
	\end{align*}
	The result follows from algebraic calculations.
\end{proof}

\begin{theorem}\label{thm:borne2}
	For any distribution $D$ on \mbox{$\X\!\times\! \Ybf$}, for any set~$\Hcalbf$ of voters from $\X$ to $\hullimY$,  
	for any prior distribution $\prior$ on $\Hcalbf$ and any $\delta\!\in\! (0,1]$, with a probability at least $1\!-\!\delta$ over the choice of the $m$-sample $S\!\sim\!(D)^m$, for every posterior distribution $\posterior$ over $\Hcalbf$ we have :
	\begin{equation*}
		\Esp_{(\xbf,\ybf)\in S}       \left( \MQs(\xbf,\ybf)\right)^2  \,\, \leq \,\, 
		\frac1m      \sum_{(\xbf,\ybf)\in S}       \left( \MQs(\xbf,\ybf)\right)^2     +  \sqrt{  \frac{2B^2}{m}  \LB 2\KL(\posterior || \prior)  +  \ln\frac{2\sqrt{m}}{\delta}\RB} \,,
	\end{equation*}
	where $B\!\in\, (0,2]$  bounds the absolute value of the margin $|\MQ(\xbf,\ybf)|$ for all $(\xbf,\ybf)$, 
	and \mbox{\small$\KL(\posterior\|\prior) \!=\! \esp{\hbf\sim \posterior}\ln\frac{\posterior(\hbf)}{\prior(\hbf)}$} is the Kullback-Leibler divergence between $\posterior$ and $\prior$. 
\end{theorem}
\begin{proof}
This proof uses many notions that are usual in classical PAC-Bayesian proofs, but the expectation over single voters is replaced with an expectation over pairs of voters. Given a distribution $D'$ on \mbox{$\X \! \times\!  \Ybf$},  for any distribution $\posterior'^2$ over $\Hcalbf$, we rewrite \mbox{\small$\Esp_{(\xbf, \ybf) \sim D'} (\MQprimes(\xbf, \ybf))^2$} as an expectation over $\posterior'^2$. 
Let $\Mbf_{\hbf,\hbf'}^{D'}$ be the r.v. such that \mbox{\small$\Esp_{(\hbf,\hbf')\sim \posterior'^2} \Mbf_{\hbf,\hbf'}^{D'}\! =\! \Esp_{(\xbf, \ybf) \sim D'} (\MQprimes(\xbf, \ybf))^2$}. 
First, we apply the Markov's inequality on the non-negative r.v. 
\mbox{\small$\Esp_{(\hbf,\hbf') \sim \prior^2} \exp	\!\left[\frac{m}{2B^2}\left(\Mbf_{\hbf,\hbf'}^S\!\! -\! \Mbf_{\hbf,\hbf'}^D\right)^2\right]$}. 
Thus, we have 
that with probability at least $1\! -\! \delta$ over the choice of $S\! \sim\! (D)^m$: 
{\small	
	\begin{align}
		\label{eq:inequality}
		\!\!\Esp_{(\hbf,\hbf') \sim \prior^2} \!\!\!\!\!\exp\!\left[\frac{m}{2B^2}	\!\left(\Mbf_{\hbf,\hbf'}^S\!\! -\! \Mbf_{\hbf,\hbf'}^D\!\right)^2\!\right]
		\! \leq\! \frac{1}{\delta} \Esp_{S\sim(D)^m} \Esp_{(\hbf,\hbf') \sim \prior^2}\!\!\!\! \!\exp\!\left[\frac{m}{2B^2}\!\left(\Mbf_{\hbf,\hbf'}^S\!\! -\! \Mbf_{\hbf,\hbf'}^D\!\right)^2\!\right].	\!\!
	\end{align}}%
	Then, we upper-bound the right-hand side of the inequality:
	\begin{small}
		\allowdisplaybreaks[1]
		\begin{align}
			\nonumber
			\Esp_{S \sim D^m}\, &\Esp_{(\hbf,\hbf') \sim \prior^2}\, \exp\!\left[\frac{m}{2B^2}\left(\Mbf_{\hbf,\hbf'}^S - \Mbf_{\hbf,\hbf'}^D\right)^2\right]\\
			\label{line:expectations} & =\ \Esp_{(\hbf,\hbf') \sim \prior^2}\, \Esp_{S \sim D^m} \exp\!\left[\frac{m}{2B^2}\left(\Mbf_{\hbf,\hbf'}^S - \Mbf_{\hbf,\hbf'}^D\right)^2\right]\\
			\nonumber & =\ \Esp_{(\hbf,\hbf') \sim \prior^2} \,\Esp_{S \sim D^m} \exp\!\left[m \, 2\left({ \frac{1}{2}\Big(1 -  \frac1{B^2}\Mbf_{\hbf,\hbf'}^S\Big) - \frac{1}{2}\Big(1 - \frac1{B^2}\Mbf_{\hbf,\hbf'}^D\Big)}\right)^2\right]\\
			\label{line:pinsker} & \leq\ \Esp_{(\hbf,\hbf') \sim \prior^2}\, \Esp_{S \sim D^m} \exp\!\left[m \, \kl\left({\frac{1}{2}\Bigg(1 -  \frac{\Mbf_{\hbf,\hbf'}^S}{B^2}\Bigg)}\middle\|{\frac{1}{2}\Bigg(1 - \frac{\Mbf_{\hbf,\hbf'}^D}{B^2}\Bigg)}\right)\right]\\
			\label{line:maurer} & \leq\ \Esp_{(\hbf,\hbf') \sim \prior^2} 2 \sqrt{m}\ =\ 2\sqrt{m}\,.
		\end{align}%
	\end{small}%
	Line~\ref{line:expectations} comes from the fact that the distribution $\prior$ is defined {\it a priori}, {\it i.e.}, before observing $S$. Since $B$ upper-bounds the absolute value of the margin, both \mbox{\small$\frac{1}{2}\Big(1\! -\!  \frac{\Mbf_{\hbf,\hbf'}^S}{B^2}\Big)$} and \mbox{\small$\frac{1}{2}\Big(1\! - \! \frac{\Mbf_{\hbf,\hbf'}^D}{B^2}\Big)$} lie between $0$ and $1$. Line~\ref{line:pinsker} is then an application of Pinsker's inequality\footnote{The Pinkser inequality is: \mbox{$2(q-p)^2\! \leq\! \kl(q\|p)\! =\! q\ln\frac{q}{p}\! +\!(1-q)\ln\frac{1-q}{1-p}$}}. Finally, Line~\ref{line:maurer} is an application of \mbox{\citet{m-04} (Theorem 5)}, which is stated to be valid for $m\geq 8$, but is also valid for any $m \geq 1$.
	
	\noindent
	By applying this upper bound in Inequality~\eqref{eq:inequality} and by taking the logarithm on each side, with probability at least $1 - \delta$ over the choice of $S \sim (D)^m$, we have:
$$
		\ln \left(\Esp_{(\hbf,\hbf') \sim \prior^2} \exp\left[\frac{m}{2B^2}\left(\Mbf_{\hbf,\hbf'}^S - \Mbf_{\hbf,\hbf'}^D\right)^2\right]\right)\ \leq\ \ln \left( \frac{2\sqrt{m}}{\delta} \right).$$
	
	Now, we need the change of measure inequality\footnote{The change of measure is an important step in most PAC-Bayesian proofs~\cite{seldin-tishby-10}.} of Lemma~\ref{lem:change_of_measure_2} (stated below) that has the novelty to use pairs of voters. By applying this lemma with \mbox{\small$\phi(\hbf,\hbf') = \frac{m}{2B^2}\left(\Mbf_{\hbf,\hbf'}^S\! -\! \Mbf_{\hbf,\hbf'}^D\right)^2$}, and by using Jensen's inequality exploiting the convexity of $\phi(\hbf,\hbf')$, we obtain that for all distributions $\posterior$ on $\Hcalbf$:%
	\begin{small}
		\begin{align*}
			\ln\! \left(\!\Esp_{(\hbf,\hbf') \sim \prior^2}\!\!\!\! \exp\!\left[\frac{m}{2B^2}\!\!\left(\Mbf_{\hbf,\hbf'}^S\! -\! \Mbf_{\hbf,\hbf'}^D\right)^2\right]\right) 
			\geq \Esp_{(\hbf,\hbf')\sim \posterior^2} \frac{m}{2B^2}\!\left(\Mbf_{\hbf,\hbf'}^S\! -\! \Mbf_{\hbf,\hbf'}^D\right)^2\!\! -\! 2 \KL(\posterior\|\prior)&\\
			\geq \frac{m}{2B^2}\!\left(\Esp_{(\hbf,\hbf')\sim \posterior^2} \Mbf_{\hbf,\hbf'}^S -\!\! \Esp_{(\hbf,\hbf')\sim \posterior^2} \Mbf_{\hbf,\hbf'}^D\right)^2 \!\!-\! 2 \KL(\posterior\|\prior)&\,.
		\end{align*}%
	\end{small}%
	From all what precedes, with probability at least $1-\delta$ on the choice of $S\! \sim\! (D)^m$, for every posterior distribution $\posterior$ on $\Hcalbf$, we have:
	\begin{align*}
		\frac{m}{2B^2}\left(\Esp_{(\xbf, \ybf) \sim S}(\MQs(\xbf, \ybf))^2 -    \Esp_{(\xbf, \ybf) \sim D}(\MQs(\xbf, \ybf))^2\right)^2     - 2\KL(\posterior\|\prior)\ \leq\ \ln \left( \frac{2\sqrt{m}}{\delta} \right).
	\end{align*}
	The result follows from algebraic calculations.
\end{proof}

The change of measure used in the previous proof is stated below.
\begin{lemma}[Change of measure inequality for pairs of voters]
	\label{lem:change_of_measure_2}
	For any set of voters $\Hcalbf$, for any distributions $\prior$ and $\posterior$ on $\Hcalbf$, and for any measurable function $\phi : \Hcalbf\! \times\! \Hcalbf \mapsto \R$, we have:$$
		\displaystyle \ln\left(\Esp_{(\hbf,\hbf')\sim \prior^2} \exp\big[\phi(\hbf,\hbf')\big]\right)\ \geq\ \Esp_{(\hbf,\hbf')\sim \posterior^2} \phi(\hbf,\hbf') - 2\KL(\prior\|\posterior)\,.$$
\end{lemma}
\begin{proof}
	The proof is very similar to the one of \mbox{\citet{seldin-tishby-10} (Lemma 4.)}, but is defined using pairs of voters.
	The first inequality below is given by using Jensen's inequality on the concave function~$\ln(\cdot)$.
		\begin{align*}
			\Esp_{(\hbf,\hbf')\sim \posterior^2} \phi(\hbf,\hbf') &= \Esp_{(\hbf,\hbf')\sim \posterior^2} \ln\!\left( \e{\phi(\hbf,\hbf')}\right) 
			= \Esp_{(\hbf,\hbf')\sim \posterior^2} \ln\! \left(\e{\phi(\hbf,\hbf')}\,\frac{\posterior^2(\hbf, \hbf')}{\prior^2(\hbf,\hbf')}  \frac{\prior^2(\hbf,\hbf')}{\posterior^2(\hbf,\hbf')}   \right) \\
			&\quad =\ \KL(\posterior^2\|\prior^2)  \, +\,  \Esp_{(\hbf,\hbf')\sim \posterior^2}  \ln\! \left( \e{\phi(\hbf,\hbf')}\, \frac{\prior^2(\hbf,\hbf')}{\posterior^2(\hbf,\hbf')}   \right)\\
			&\quad \leq\ \KL(\posterior^2\|\prior^2)\,+\, \ln\! \left(\Esp_{(\hbf,\hbf')\sim \posterior^2} \e{\phi(\hbf,\hbf')}\,\frac{\prior^2(\hbf,\hbf')}{\posterior^2(\hbf,\hbf')}   \right) \\
			&\quad \leq\ \KL(\posterior^2\|\prior^2) + \ln\! \left( \Esp_{(\hbf,\hbf')\sim \prior^2} \e{\phi(\hbf,\hbf')} \right)\\
			&\quad =\ 2\KL(\posterior\|\prior) + \ln\! \left( \Esp_{(\hbf,\hbf')\sim \prior^2} \e{\phi(\hbf,\hbf')} \right).
		\end{align*}
	Note that the last inequality becomes an equality if $\posterior$ and $\prior$ share the same support. The last equality comes from the definition of the KL-divergence, and from the fact that $\prior^2(\hbf, \hbf') = \prior(\hbf)\prior(\hbf')$ and $\posterior^2(\hbf, \hbf') = \posterior(\hbf)\posterior(\hbf')$.
\end{proof}

\bibliography{biblio_abbregee}
\bibliographystyle{icml2011}

\end{document}